\documentclass{article}
\usepackage{ijcai22}

\usepackage{times}
\usepackage{multirow}
\usepackage{soul}
\usepackage{url}
\usepackage[hidelinks]{hyperref}
\usepackage[utf8]{inputenc}
\usepackage[small]{caption}
\usepackage{graphicx}
\usepackage{amsmath}
\usepackage{amsthm}
\usepackage{amssymb}
\usepackage{booktabs}
\usepackage{bm}
\usepackage{dsfont}
\usepackage{tikz}
\usepackage{forest}
\usepackage[ruled, vlined, linesnumbered, noend]{algorithm2e}
\DontPrintSemicolon
\SetKwInput{KwInput}{Input}
\usepackage{algorithmic}
\usepackage{thmtools}
\usepackage{thm-restate}
\usepackage{subcaption}

\newtheorem{theorem}{Theorem}
\newtheorem{lemma}{Lemma}
\newtheorem{corollary}[theorem]{Corollary}
\newtheorem{proposition}{Proposition}
\newtheorem{definition}{Definition}

\newcommand{\marmax}{\texttt{MAR}_{\max}}
\newcommand{\AC}{\mathcal{T}}
\newcommand{\BN}{\mathcal{N}}
\newcommand{\SPN}{\mathcal{S}}

\newcommand{\CREDAL}{\bm{\mathcal{C}}}
\newcommand{\ordering}{\sigma}
\newcommand{\evidence}{e}

\newcommand{\acnode}{t}
\newcommand{\subcirc}{\alpha}
\newcommand{\parvals}{\bm{u}}
\newcommand{\param}{\theta}
\newcommand{\indicators}{\bm{\lambda}}
\newcommand{\parameters}{\Theta}
\newcommand{\val}{v}

\newcommand{\EX}{\texttt{EX}}

\title{Robustness Guarantees for Credal Bayesian Networks via Constraint Relaxation over Probabilistic Circuits}

\author{
Hjalmar Wijk\and
Benjie Wang\and
Marta Kwiatkowska
\affiliations
University of Oxford\\
\emails
hjalmar.wijk@st-annes.ox.ac.uk,
benjie.wang@keble.ox.ac.uk,
marta.kwiatkowska@cs.ox.ac.uk
}

\begin{document}

\maketitle

\begin{abstract}
  
  In many domains, worst-case guarantees on the performance (e.g., prediction accuracy) of a decision function subject to distributional shifts and uncertainty about the environment are crucial. In this work we develop a method to quantify the robustness of decision functions with respect to credal Bayesian networks, formal parametric models of the environment where uncertainty is expressed through \textit{credal sets} on the parameters. In particular, we address the maximum marginal probability ($\texttt{MAR}_{\max}$) problem, that is, determining the greatest probability of an event (such as misclassification) obtainable for parameters in the credal set. We develop a method to faithfully transfer the problem into a constrained optimization problem on a probabilistic circuit. By performing a simple constraint relaxation, we show how to obtain a guaranteed upper bound on $\texttt{MAR}_{\max}$ in linear time in the size of the circuit. We further theoretically characterize this constraint relaxation in terms of the original Bayesian network structure, which yields insight into the tightness of the bound. We implement the method and provide experimental evidence that the upper bound is often near tight and demonstrates improved scalability compared to other methods. 

\end{abstract}

\section{Introduction}

Probabilistic models allow us to make quantitative inferences about the behaviour
of complex systems, and are an important tool to guide their use and design. When such models are learnt from data, exposed to potential distribution shifts or are partially unknown, it is important to be
able to verify the robustness of inferences on the model to these uncertainties. This is particularly relevant for decision functions taking action in the model, where much work has gone into verifying worst-case behaviour when exposed to various disturbances or
changes in the environment (distribution shifts).
Causal Bayesian networks (BNs)~\cite{pearl1985bayesian} are compelling models for this purpose, since one can perform causal interventions
on them, giving rise to families of distributions that share a common
structure. However, performing useful inference on BNs is often intractable, and one way to address this is to compile them into more tractable representations such as arithmetic circuits~\cite{darwiche2003differential}. Recent work has shown that such compilation methods can also efficiently compute bounds on a decision function's
robustness to causal interventions~\cite{wang2021provable}. A limiting factor on the applicability of these methods is
the need to have an exact model, where all non-intervened parameters are known precisely. This is
difficult to achieve when learning parameters from data, since most settings will only allow reliable
determination up to some error bound $\epsilon$.

In this paper we study robustness of credal Bayesian networks (CrBNs) \cite{maua2020thirty}, a generalisation of Bayesian networks where parameters are only known to be within some
credal sets (e.g., intervals). They can be used to model causal interventions, but are also very well suited to modelling parameters learned from data, as well as %and can be used to model 
modelling of exogenous variables \cite{zaffalon2020structural}.

We consider the maximum marginal probability ($\texttt{MAR}_{\max}$) problem for CrBNs and develop a solution by encoding the network as a tractable probabilistic circuit (a credal extension of sum-product networks, called CSPNs). %
More specifically, this paper makes the following contributions: 
(i) a method for constructing a probabilistic circuit whose parameters semantically represent the conditional probability distributions of a BN, allowing the transfer of credal inference problems from a highly intractable setting (CrBNs) to a tractable one (CSPN) through constraint relaxation;
(ii) algorithms which make use of this transfer to compute upper and lower bounds on probabilities of events under many forms of parameter uncertainty;
(iii) a characterization of the tightness of the upper bound in terms of the network structure; and
(iv) an evaluation on a set of benchmarks, demonstrating comparable precision and significantly improved scalability compared to state-of-the-art credal network inference, while also providing formal guarantees.

    Due to space constraints some details and proofs can be found in the Appendix of the extended paper at  \hyperlink{http://www.fun2model.org/bibitem.php?key=WWK22}{http://www.fun2model.org/bibitem.php?key=WWK22}

\subsection{Related Work}

The problem of robustness of inferences under imprecise knowledge of the distribution has been studied under many guises. In the machine learning community, there has been much work on robustness of classifiers to simple adversarial attacks or distribution shifts \cite{quinonero2009dataset,zhang2015multi,lipton2018detecting}. Motivated by safety concerns, methods have been developed to compute formal guarantees of robustness through constraint solving~ \cite{katz2017reluplex,narodytska2018verifying} or output reachability analysis~\cite{ruan2018reachability}. However, these methods do not model the environment, and are thus limited in the types of distributional shifts they can address.

In the Bayesian network literature, robustness has primarily been studied in terms of the effect of parameters on inference queries, such as marginal probabilities. For instance, sensitivity analysis \cite{coupe2000sensitivity,chan2004multiple} is concerned with the effect of small, local changes/perturbations to parameters. Closer to our work is the formalism of credal networks \cite{maua2020thirty}, which represent imprecise knowledge by positing sets of parameters for each conditional distribution, rather than precise values. Inference then corresponds to computing maximal (or minimal) probabilities over the possible parameter values. Unfortunately, exact methods for inference in credal networks do not perform well except for smaller networks with simple credal sets, or in special cases such as polytrees \cite{fagiuoli19982u,decampos2007intprog}. On the other hand, approximate methods \cite{CANO2007hillclimb,antonucci2010gl2u,antonucci2015approxlp} usually cannot provide theoretical guarantees (upper bounds), limiting their applicability in safety-critical scenarios. 

This paper builds on work %that has been done
showing the tractability of credal inference for certain probabilistic circuits \cite{maua2017credal} \cite{mattei2020tractable}. Our key contribution is a method for mapping credal network problems into tractable credal inference problems on probabilistic circuits, which affords not only greater scalability compared to the state-of-the-art in credal network inference, but also provides formal guarantees.

Finally, methods for providing robustness guarantees for classifiers in combination with a Bayesian network environment model have recently been proposed \cite{wang2021provable}. Our paper generalizes and extends their work, enabling efficient computation for broader and more realistic classes of parameter uncertainty.

\section{Background}

A Bayesian network (BN) $\mathcal{N}=(\mathcal{G},\Theta)$ over discrete variables $\bm{V}=\{V_1,...,V_n\}$ consists of a directed acyclic graph (DAG) $\mathcal{G}=(\bm{V},\bm{E})$ and a set of parameters $\Theta$. It is a factoring of a joint probability distribution $p_\mathcal{N}$ into conditional distributions for each variable, such that
$$p_\mathcal{N}(V_1,...,V_n)= \prod_{i=1}^n p_\mathcal{N}(V_i\vert \text{pa}(V_i)),$$
where the parents $\text{pa}(V_i)$ of $V_i$ are the set of variables $V_j$ such that $(V_j,V_i)\in \bm{E}$. $\Theta$ is the set of parameters of the form:
$$\theta_{v_i\vert \bm{u}_i}=p_{\mathcal{N}}(V_i=v_i \vert \bm{U}_i=\bm{u}_i),$$ 
    for each instantiation $v_i,\bm{u}_i$ of a variable $V_i$ and its parents $\bm{U}_i$.

Given a Bayesian network model, to obtain useful information about the distribution we will need to perform inference. For example, we might wish to obtain the probability $p_{\mathcal{N}}(\bm{W}=\bm{w})$ for some subset of variables $\bm{W} \subseteq \bm{V}$, a procedure known as marginalization. In the worst case, marginal inference in Bayesian networks is known to be \#P-complete, though many practical inference methods exist. % for inference in BNs.

Given a classifier, we can represent its input-output behaviour using a decision function $F:\bm{X}\to Y$, which observes some subset $\bm{X}\subseteq \bm{V}$ and tries to predict $Y \in \bm{V}$. To combine this with a Bayesian network environment model $\mathcal{N}$, we follow \cite{wang2021provable} in the construction of an \emph{augmented BN} $\mathcal{N}_F$, which is a Bayesian network based on $\mathcal{N}$ where an additional variable (node) $\hat{Y}$ is added with $\text{pa}(\hat{Y})=\bm{X}$ and $p_{\mathcal{N}_F}(\hat{Y}=\hat{y}\vert\bm{X}=\bm{x})=\mathds{1}[\hat{y}=F(\bm{x})]$. $\mathcal{N}_F$ is thus a unified model of environment and decision maker, and inference on the model can answer questions such as the prediction accuracy $p_{\mathcal{N}_F}(\hat{Y}=Y)$.

\section{Robust Inference on Bayesian Networks}

It is rarely the case that we can specify the parameters of a Bayesian network with complete certainty before performing inference. Firstly, whether the parameters are learned from data or elicited from expert knowledge, the knowledge that we obtain regarding the parameters is typically imprecise, specified as sets or intervals. Secondly, when the Bayesian network is imbued with a causal interpretation, one is often concerned about potential distribution shift, modelled by causal interventions, and their effect on inference queries.

As a running example, consider a fictional scenario depicted in Figure \ref{fig:treatment}, where patients are infected by some unobservable strain $S$ of a disease, with some strains much more severe than others, and a decision rule $F$ must be created based on observable symptoms $V$ and test results $R$ that decides whether to administer an expensive treatment option. 
While it is desirable to save resources by only administering treatment for the more severe strains, 
the result of denying treatment to a patient with a severe strain would be disastrous, so a guarantee is needed that the decision rule has a robustly low probability of an erroneous decision. 
To provide such a guarantee we model the system as an augmented BN $\mathcal{N}_F$ over variables $\{S,V,R,T\}$, where $T$ is binary and deterministic, given by $F(v,r)$. However, we do not have precise knowledge over the parameters of the BN, so we instead design intervals which specify the range of values a parameter could take. 
\begin{figure}
    \centering
    \begin{tikzpicture}[node distance={25mm}, thick, main/.style = {draw, circle}]
    \node[main, minimum size=1.2cm, align=center] (1) {\footnotesize Strain \\ \footnotesize $S$}; 
    \node[main, minimum size=1.2cm,align=center] (2) [above right of=1] {\footnotesize Symptoms \\ \footnotesize $V$};
    \node[main, minimum size=1.2cm,align=center] (3) [below right of=1] {\footnotesize Test \\ \footnotesize Results\\ \footnotesize $R$};
    \node[main, minimum size=1.2cm,align=center] (4) [below right of=2] {\footnotesize Treatment \\ \footnotesize $T$};
    
    \draw[->] (1) -- (2);
    \draw[->] (1) -- (3);
    \draw[->] (2) -- (4);
    \draw[->] (3) -- (4);
    \end{tikzpicture} 
    \caption{The DAG of an augmented BN modelling a simple fictional medical treatment scenario.}
    \label{fig:treatment}
\end{figure}
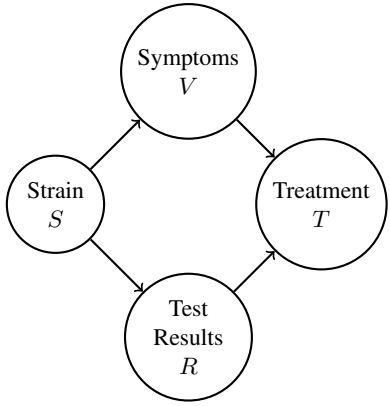

We start with $\theta_S$, the distribution over the strains. We expect that the decision rule will be deployed across a variety of areas and times, and as such we are concerned about distributional shifts in $\theta_S$. We could thus decide to allow any probability distribution across the strains, i.e. $\theta_{S=s} \in [0, 1]$. We imagine the tests used are very well understood, and we know $\theta_{R \vert s}$ exactly. Moving onto the parameters $\theta_{V \vert s}$, describing the symptoms of a particular strain, we might expect that these, unlike $\theta_S$, are relatively fixed across different settings. However, gathering enough data on each strain and symptom combination to be certain of this (fixed) parameter value might turn out to be challenging. In this case it might be suitable to take mean estimates of the parameter values $\theta^*_{V=v \vert s}$, and then select some confidence interval $[\theta^*_{V=v \vert s}-\epsilon,\theta^*_{V=v \vert s}+\epsilon]$.

\subsection{Credal Bayesian Networks}

To formalize the prior discussion, we use credal Bayesian networks \cite{maua2020thirty}, a framework that encompasses both causal interventions and imprecise knowledge of parameters.

\begin{definition}
Let $p_\Theta, \Theta \in \bm{\Theta}$, be any parameterised probability distribution, where $\bm{\Theta}$ is the set of allowed parameter values. Then we call $\bm{\mathcal{C}}\subseteq \bm{\Theta}$ a \emph{credal set} for this parameterisation, and the \emph{credal family} $\mathbb{C}_p[\bm{\mathcal{C}}]=\{p_{\Theta'}\vert \Theta'\in \bm{\mathcal{C}}\}$ is the family of distributions where the parameters are in $\bm{\mathcal{C}}$. 
\end{definition}

This is a maximally expressive formalism for credal uncertainty. However, an independence assumption between the uncertainty of different conditional distributions in a BN (sometimes known as the strong extension \cite{COZMAN2000199}) is usually assumed:

\begin{definition} \cite{COZMAN2000199}
A Credal BN (CrBN) $\mathbb{C}_{\mathcal{N}}[\bm{\mathcal{C}}]=\{\mathcal{N}_\Theta\ | \Theta \in \bm{\mathcal{C}}\}$ over a BN $\mathcal{N}_{\Theta}=(\mathcal{G},\Theta)$ is a credal family satisfying 
$$\bm{\mathcal{C}}= \prod_{V_i,\bm{u_i}} \mathcal{C}_{V_i\vert \bm{u_i}},$$
i.e. the credal set decomposes as a cartesian product of separate credal sets for each variable $V_i$ and instantiation of its parent variables $\bm{u_i}$.
\end{definition}

Since augmented Bayesian networks are simply Bayesian networks with an additional deterministic node (the decision function), we can convert any credal set over a Bayesian network model $\mathcal{N}$ to a credal set over $\mathcal{N}_F$ by maintaining the credal sets for all variables, while assuming the conditional distribution for the new variable $\hat{Y}$ is known exactly. This framework then fully generalizes the ``interventional robustness problem'' introduced by \cite{wang2021provable} to allow arbitrary credal sets for parameters; see Appendix for details.

\subsection{Problem Definition}
\label{sec:treatment}
In the treatment example we wished to guarantee the worst-case probability of an event occurring over a CrBN. We will now formalise this problem.   

\begin{definition}
Given an (augmented) CrBN $\mathbb{C}_{\mathcal{N}}[\bm{\mathcal{C}}]$

and an event $e$ (an instantiation of a subset of the variables), the maximum marginal probability ($\texttt{MAR}_{\max}$) problem is that of determining 
$$\texttt{MAR}_{\max}(\mathcal{N},\bm{\mathcal{C}},e)=\max_{\Theta \in \bm{\mathcal{C}}}p_{\mathcal{N}_{\Theta}}(e).$$

\end{definition}

This generalization of causal interventions enables many new problems to be considered, as causal interventions require parameters to be known exactly or be entirely unknown. Crucially, it allows us to model parameters which are estimated from data to be within some interval. It also allows the degree of uncertainty to depend on the value of the parents, as it might if some parent values are rare and lack data points.

As an illustration we now define a CrBN over $\mathcal{N}_F$ to formalize the treatment scenario in Figure \ref{fig:treatment}. We imagine there are three strains $s_1,s_2,s_3$, of which only $s_3$ is severe and requires treatment. We take $V$ and $R$ to be binary variables (symptomatic/asymptomatic and positive/negative test). We wish to be able to apply the decision rule in any situation where the prevalence of $s_3$ is at most $0.1$, so we assign $\mathcal{C}_S=\{\bm{\theta} \in Z_3\vert \theta_{S=s_3} <0.1\}$, where $Z_3$ is the three-dimensional probability simplex. We use singleton credal sets for $R$ and confidence intervals for $V$, with the values given in Table \ref{tb:credalsets}. The decision rule to be analysed gives treatment when $R=V$, since this is unlikely for $s_1$ and $s_2$.

\begin{table}
\footnotesize
\centering
\begin{tabular}{llll}
\toprule
                         & $s_1$                      & $s_2$                       & $s_3$   \\
                         \midrule
$\theta_R$ & 0.95                      & 0.05                       & 0.5   \\
$\theta_V$ & $0.2\pm 0.1$ & $0.8 \pm 0.1$ & $0.6 \pm 0.2$ \\
\bottomrule
\end{tabular}
\caption{Credal sets for symptom and test result parameters.}
\label{tb:credalsets}
\end{table}

This is an instance of the maximum marginal probability $\marmax$ problem, where the CrBN $\mathbb{C}_{\mathcal{N}}[\bm{\mathcal{C}}]$ is as specified above, and the event of interest is $e = (T = 0) \wedge (S = s_3)$.

\section{Credal Robustness via Probabilistic Circuits}

In this section, we present an efficient method for bounding $\marmax$ credal robustness for Bayesian networks with guarantees. In particular, the method returns an upper bound on $\marmax$. In the treatment example, this would mean that we can be certain that the probability of denying treatment to a patient with the severe strain does not exceed the computed value, assuming all parameters lie within the credal sets.  Our method is based upon establishing a correspondence between credal BNs and credal sum-product networks (CSPN) \cite{maua2017credal}, a recently proposed model which introduces uncertainty sets over the weights of a sum-product network. In particular, we develop an algorithm for \emph{compiling} CrBNs into equivalent CSPNs. By efficiently solving a similar credal maximization problem on the CSPN, we can derive upper bounds on $\marmax$ for the original CrBN. 

\subsection{Compilation to Arithmetic Circuits}

The first step of our method is to compile the credal Bayesian network to an arithmetic circuit. To describe this, we first consider an alternative representation of a Bayesian network. 

\begin{definition}
\cite{darwiche2003differential}
The \emph{network polynomial} of a BN $\mathcal{N}$ is defined as
$$l_{N}[\bm{\lambda},\Theta]=\sum_{v_1,...,v_n}\prod_{i=1}^n \theta_{v_i\vert u_i} \lambda_{v_i},$$
where $\lambda_{v_i}$ are indicator variables for variable $V_i$, which take the value 1 if $V_i = v_i$ and 0 otherwise.
\end{definition}

The network polynomial is a multilinear function which unambigously encodes the graphical structure of the Bayesian network, for any value of the parameters $\Theta$. In particular, one can obtain the joint probability $p_{\mathcal{N}}(v_1, ..., v_n)$ for any instantiation $v_1, ..., v_n$ by setting the indicator variables and evaluating the network polynomial. Unfortunately, it has an exponential number of terms in the number of variables of the BN, which means we cannot use it directly. The goal of compilation is to represent the network polynomial more efficiently, by exchanging sums and products where possible. The result of such a procedure can be interpreted as a rooted directed acyclic graph (DAG) called an arithmetic circuit.

\begin{definition}\cite{darwiche2003differential}
An arithmetic circuit (AC) $\mathcal{T}$
over variables $\bm{V}$ and parameters $\Theta$ is a rooted DAG, whose internal nodes are labelled
with $+$ or $\times$ and whose leaf nodes are labelled with indicator
variables $\lambda_v$

or non-negative parameters.
For an internal node $t$ we will write $\mathcal{T}_t$ for the arithmetic circuit containing $t$ and all its descendants. 
\end{definition}

\begin{definition}\cite{chan2006robustness}
A complete subcircuit
$\alpha$ of an AC is obtained by traversing the circuit top-down,
choosing one child of every visited $+$-node and all children
of every visited $\times$-node. The term $\emph{term}(\alpha)$ of $\alpha$ is the product of all leaf nodes visited (i.e. all indicator and parameter variables). The AC polynomial $l_{\mathcal{T}}[\bm{\mathcal{\lambda}},\Theta]$ is the sum of the terms of all complete subcircuits. 

\end{definition}

Compilation will produce an AC $\mathcal{T}$ which has the same polynomial as the BN, i.e. $l_{N}[\bm{\lambda},\Theta] = l_{\mathcal{T}}[\bm{\mathcal{\lambda}},\Theta]$. In addition, it will satisfy technical conditions called \emph{decomposability}, \emph{determinism} and \emph{smoothness}, which allow us to perform many inference queries in linear time in the size of the circuit.

In \cite{wang2021provable} a method is described for compiling an augmented BN to a smooth, decomposable and deterministic AC, which allows one to tractably compute marginal probabilities involving both the decision function and Bayesian network variables. In order to support further queries, they additionally impose ordering constraints on the AC.

\begin{definition}
\label{def:split}
 A $+$-node $t$ with children $t_1,...,t_n$ in an arithmetic circuit $\AC$ \emph{splits} on variable $V_i$ if there exists an ordering of the domain $v_i^1,...,v_i^n$ of $V_i$ such that all complete subcircuits of $\AC_{t_i}$ contain the indicator $\lambda_{v_i}$.

\end{definition}

\begin{definition}
\label{def:topord}
Let $\ordering = (V_1, ..., V_n)$ be a topological ordering of the variables in BN $\BN$. We say that an (smooth, decomposable, deterministic) arithmetic circuit $\AC$ computes the BN $\BN$ respecting $\ordering$ if:
\begin{enumerate}
    \item $l_{\mathcal{T}}[\bm{\lambda},\Theta]=l_\mathcal{N}[\bm{\lambda},\Theta]$
    \item Each $+$-node in $\mathcal{T}$ splits on some variable $V_i$. We define $\text{split}(t)$ for a $+$-node $t$ to be the variable it splits on.
    \item The variables are split respecting the topological order. That is, if
    $V_i =\text{split}(t), V_j=\text{split}(t')$, then $$t' \text{ is a descendant of }t \Longrightarrow j >i.$$ 
\end{enumerate}
\end{definition}

In other words, it is required that the AC represents the same polynomial as the BN, and further that the AC satisfies particular structural constraints that mean that the AC must split on parents before children. This leads to the following new result, which intuitively means that, when an AC splits on variable $V$, the values of its parents are already known.

\begin{restatable}{lemma}{lemParents}
\label{lem:parents}
Suppose that $\AC$ computes $\BN$ respecting some topological order. Let $\acnode$ be a $+$-node in $\AC$ splitting on some variable $V$. Then all complete subcircuits $\subcirc$ which include $\acnode$ must agree on the value of its parents $\text{pa}_{\mathcal{N}}(V)$.
\end{restatable}

The AC compiled from the treatment example (Figure \ref{fig:treatment}) is too large to include in its entirety, but Figure \ref{fig:acexample} shows one branch from the root sum node, with $+$-nodes labelled with the variable they split on. Notice that the topological order $(S, R, V, T)$ is respected, and that, at every $+$-node, the value of the parents of the splitting variable are already ``known''.

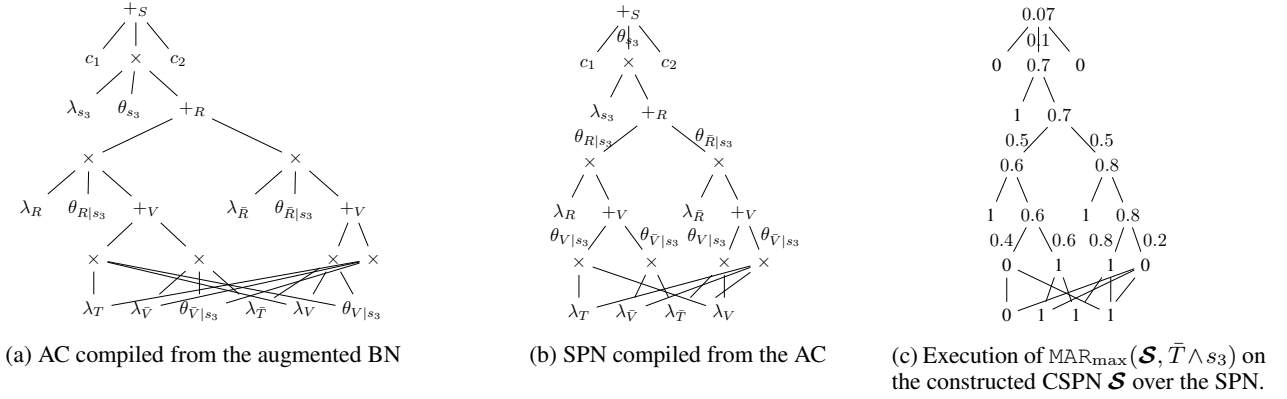
\begin{figure*}[t]
    \centering
    \begin{subfigure}[t]{0.4\textwidth}
        \centering
        \scalebox{0.7}{
        \begin{forest}
for tree={l sep=5pt}
[$+_S$
    [$c_1$
    ]
    [$\times$  
      [$\lambda_{s_3}$] 
      [$\theta_{s_3}$]
      [$+_R$
        [$\times$
            [$\lambda_R$]
            [$\theta_{R\vert s_3}$]
            [$+_V$
                [$\times$,name=rvbranch
                    [$\lambda_{T}$,name=t]
                ]
                [$\times$
                    [$\lambda_{\bar{V}}$,name=nv]
                    [$\theta_{\bar{V}\vert s_3}$,name=pnv]
                    [$\lambda_{\bar{T}}$,name=nt]
                ]
            ]
        ]
        [$\times$
            [$\lambda_{\bar{R}}$]
            [$\theta_{\bar{R}\vert s_3}$]
            [$+_V$
                [$\times$,name=vbranch
                    [$\lambda_V$,name=v]
                    [$\theta_{V\vert s_3}$,name=pv]
                ]
                [$\times$,name=nvbranch]
            ]
        ]
      ]
    ]
    [$c_2$
    ]
]
\draw (rvbranch) -- (v);
\draw (rvbranch) -- (pv);
\draw (vbranch) -- (nt);
\draw (nvbranch) -- (nv);
\draw (nvbranch) -- (pnv);
\draw (nvbranch) -- (t);
\end{forest}
        }
        \caption{AC compiled from the augmented BN}
        \label{fig:acexample}
    \end{subfigure}%
    ~ 
    \begin{subfigure}[t]{0.29\textwidth}
        \centering
        \scalebox{0.7}{
            \begin{forest}
for tree={l sep=5pt}
[$+_S$
    [$c_1$
    ]
    [$\times$,edge label={node[midway] {$\theta_{s_3}$}}  
      [$\lambda_{s_3}$] 
      [$+_R$
        [$\times$,edge label={node[midway,left] {$\theta_{R\vert s_3}$}}
            [$\lambda_R$]
            [$+_V$
                [$\times$,name=rvbranch,edge label={node[midway,left] {$\theta_{V\vert s_3}$}}
                    [$\lambda_{T}$,name=t]
                ]
                [$\times$,edge label={node[midway,right] {$\theta_{\bar{V}\vert s_3}$}}
                    [$\lambda_{\bar{V}}$,name=nv]
                    [$\lambda_{\bar{T}}$,name=nt]
                ]
            ]
        ]
        [$\times$,edge label={node[midway,right] {$\theta_{\bar{R}\vert s_3}$}}
            [$\lambda_{\bar{R}}$]
            [$+_V$
                [$\times$,name=vbranch,edge label={node[midway,left] {$\theta_{V\vert s_3}$}}
                    [$\lambda_V$,name=v]
                ]
                [$\times$,name=nvbranch,edge label={node[midway,right] {$\theta_{\bar{V}\vert s_3}$}}]
            ]
        ]
      ]
    ]
    [$c_2$
    ]
]
\draw (rvbranch) -- (v);
\draw (vbranch) -- (nt);
\draw (nvbranch) -- (nv);
\draw (nvbranch) -- (pnv);
\draw (nvbranch) -- (t);
\end{forest}
        }
        \caption{SPN compiled from the AC}
        \label{fig:spnexample}
    \end{subfigure}
     ~ 
    \begin{subfigure}[t]{0.28\textwidth}
        \centering
        \scalebox{0.7}{
        \begin{forest}
for tree={l sep=5pt}
[$0.07$
    [0
    ]
    [$0.7$,edge label={node[midway] {$0.1$}}  
      [1] 
      [$0.7$
        [$0.6$,edge label={node[midway,left] {$0.5$}}
            [1]
            [$0.6$
                [0,name=rvbranch,edge label={node[midway,left] {$0.4$}}
                    [0,name=t]
                ]
                [1,edge label={node[midway,right] {$0.6$}}
                    [1,name=nv]
                    [1,name=nt]
                ]
            ]
        ]
        [$0.8$,edge label={node[midway,right] {$0.5$}}
            [1]
            [$0.8$
                [1,name=vbranch,edge label={node[midway,left] {$0.8$}}
                    [1,name=v]
                ]
                [0,name=nvbranch,edge label={node[midway,right] {$0.2$}}]
            ]
        ]
      ]
    ]
    [0
    ]
]
\draw (rvbranch) -- (v);
\draw (vbranch) -- (nt);
\draw (nvbranch) -- (nv);
\draw (nvbranch) -- (pnv);
\draw (nvbranch) -- (t);
\end{forest}
        }
        \caption{Execution of $\texttt{MAR}_{\max}(\bm{\mathcal{S}},\bar{T}\land s_3)$ on the constructed CSPN $\bm{\mathcal{S}}$ over the SPN. 
        }
        \label{fig:spnmaximization}
    \end{subfigure}
    \caption{Illustration of Algorithm \ref{alg:upper} for the treatment example. Due to space constraints we show only the $S=s_3$ branch of the AC/SPN.}
\end{figure*}

\subsection{Compiling to Credal SPNs}

While this compiled AC allows us to efficiently compute marginals for given parameter values $\Theta$, it does not effectively represent credal sets, and thus finding maximizing parameter values is challenging (one would need to solve constraints potentially spread out across the whole circuit). 

In the next step of our method, we further compile the AC to a sum-product network (SPN). SPNs differ from ACs in that they lack parameter nodes and instead have parameters (i.e. weights) associated with branches from sum nodes.

\begin{definition} \cite{poon2011spn}
A sum-product network  (SPN) over variables $\bm{V}$ and with weights $W$ is a rooted DAG whose internal nodes are labelled with either $+$ or $\times$, and whose leaf nodes are labelled  with indicator variables $\lambda_v$. The branches of a sum node $t_i$ with $k$ branches are labelled with weights $w_{i,1},..,w_{i,k}$. 
\end{definition}

\begin{definition}
A complete subcircuit $\alpha$ of an SPN $S$ is obtained by traversing the circuit top-down,
choosing one child of every visited $+$-node and all children
of every visited $\times$-node. 
The term $\emph{term}(\alpha)$ of $\alpha$ 
is the product of all leaf nodes visited (i.e. all indicators) 
and all weights $w_{ij}$ along branches chosen by the subcircuit.

The SPN-polynomial $l_{\mathcal{S}}[\bm{\mathcal{\lambda}},W]$ is the sum of the terms of all complete subcircuits.
\end{definition}

Our compilation differs from that presented in \cite{rooshenas2014learning} in that we make use of the particular structure of the AC, shown in Lemma \ref{lem:parents}, to make sure the weights on the sum nodes directly correspond to the parameters in the BN (which would not be the case under standard compilation). 

At a high level, the compilation only involves two steps:

\begin{enumerate}
    \item For each sum node $t$ splitting on $V_i$, assign weights over branches according to $\bm{\theta}_{V_i\vert u_i}$, where all variables in $u_i$ are known due to Lemma \ref{lem:parents}.
    \item Remove all parameter nodes. 
\end{enumerate}

To algorithmically decide which parameters correspond to a particular sum node we construct a notion of 'possible values' for variables at nodes in the SPN. We first say that a node 'conditions' on $V=v$ if the node is a parameter node $\theta_{W=w|V=v,\bm{U}=\bm{u}}$, and define $$\bm{P}_t(V)=\{v: \exists t' \in \text{descendants}(t), t' \text{ conditions on } V=v \} .$$ \begin{corollary}
For any sum node $t$ splitting on $V$, if $W$ is a parent of $V$ then $\bm{P}_t(W)$ must contain exactly one possible value. 
\end{corollary}
\begin{proof}
Any complete subcircuit corresponds to a term of the network polynomial, and must contain a parameter $\theta_{V|w_i,\bm{u_i}}$ for some $w_i,\bm{u_i}$. This parameter cannot occur as an ancestor of $t$, since it would then be impossible to satisfy Definition \ref{def:split}, and so it must be a descendant. Thus $\bm{P}_t(V)$ is non-empty. By Lemma 1, it cannot contain more than 1 element. 
\end{proof}

These sets uniquely determine the values of all parents, and thus which parameters to use. The sets can be efficiently computed , as described in Algorithm \ref{alg:SPN}.

Figure \ref{fig:spnexample} shows the result for the AC in Figure \ref{fig:acexample}.

\begin{algorithm}
\footnotesize

\SetAlgoLined
\KwInput{AC $\mathcal{T}$ computing $\mathcal{N}$ and satisfying Definition \ref{def:topord}.
}
\KwResult{An SPN computing $\mathcal{N}$ where all sum node weights correspond to a CPT $\bm{\theta}_{V_i|u_i}$}
\Begin{
For indicator nodes $t$, assign $\bm{P}_t(V)=\{\}$; \\
For parameter nodes $\theta_{v|u_1,u_2,...,}$, assign $\bm{P}_t(U_i)=\{u_i\}$ if $U_i\in \text{pa}(V)$ and $\bm{P}_t(U_i)=\{\}$ otherwise; \\
For inner nodes $t$ compute $\bm{P}_t(V)=\bigcup_{c\in \text{children}(t)} \bm{P}_{c}(V)$; \\
For sum nodes $t$, label the edges according to $\bm{\theta}_{V|u_1,u_2,...}$, where $t$ splits on $V$ and $u_i$ is the unique value in $\bm{P}_t(U_i)$; \\
Remove all parameter nodes. \\
}
\caption{SPN compilation from AC}
\label{alg:SPN}
\end{algorithm}

\begin{restatable}{proposition}{propSPN}
\label{prop:SPNeq}
Given an arithmetic circuit $\mathcal{T}$ which computes $\BN$ satsifying some topological order, the SPN $\mathcal{S}$ compiled as above satisfies $$l_{\mathcal{T}}[\bm{\lambda},\Theta]=l_{\mathcal{S}}[\bm \lambda, \Theta].$$
\end{restatable}

\begin{proof}
We can put the complete subcircuits of $\mathcal{T}$ and $\mathcal{S}$ in a one-to-one correspondence by the choice of branch at each sum node (since only the parameter nodes/weights have changed). Let $\alpha_T$ be a subcircuit of $\AC$, and $\alpha_S$ the corresponding subcircuit of $\SPN$. For every variable $V$, $\alpha_T$ contains exactly one $+$-node splitting on $V$, and exactly one parameter of the form $\theta_{V|pa(V)}$. The compilation procedure moves this parameter to be a weight of the $+$-node splitting on $V$, so that the overall term is unchanged. Applying this to all variables, we have that $\text{term}(\alpha_T) = \text{term}(\alpha_S)$, and thus the result.
\end{proof}

For credal families over SPNs satisfying an independence requirement between all sum nodes (such that knowing the weights of one sum node does not affect your uncertainty over other weights), $\texttt{MAR}_{\max}$ can be computed efficiently. These are exactly the Credal SPNs introduced in \cite{maua2017credal}. 

\begin{definition} \cite{maua2017credal}
A Credal SPN (CSPN) is a credal family $\mathbb{C}_S[\bm{\mathcal{C}}]$ over an SPN $S$ satisfying 
$$\bm{\mathcal{C}}= \prod_{i=1}^n \mathcal{C}_{i},$$
where $\mathcal{C}_i$ is a subset of a probability simplex on the weights of sum node $i$.
\end{definition}

We can construct a credal family over the compiled SPN, which is equivalent to our CrBN, by requiring all sum nodes that split on a variable $V_i$ and have parents $\bm{u_i}$ to \textit{(i)} all have the same weights and \textit{(ii)} have that weight be in $\mathcal{C}_{V_i \vert u_i}$. However, this will not in general be a CSPN, since $(i)$ breaks the independence requirement (observing the weights of one sum node will change the credal set for a different sum node if they both split on the same variable with the same values of the parents). 

We can, however, construct a CSPN by removing requirement $(i)$, which creates a strictly larger credal family. This relaxation is the final step of our compilation process.

\begin{lemma}
\label{lem:bound}
For a CrBN $\mathbb{C}_{\mathcal{N}}[\bm{\mathcal{C}}_{\mathcal{N}}]$ and its compiled CSPN $\mathbb{C}_{\mathcal{S}}[\bm{\mathcal{C}}_{\mathcal{S}}]$,
$$\max_{\Theta\in \bm{\mathcal{C}}_\mathcal{S}} l_S[\bm{\lambda},\Theta] \geq \max_{\Theta' \in \bm{\mathcal{C}}_\mathcal{N}}l_{\mathcal{N}}[\bm{\lambda},\Theta'].$$
\end{lemma}
\begin{proof}
For any given $\Theta\in\bm{\mathcal{C}}_{\mathcal{N}}$ we have $l_S[\bm{\lambda},\Theta]=l_{\mathcal{N}}[\Theta,\bm{\lambda}]$, by Proposition \ref{prop:SPNeq} and the fact that $\mathcal{S}$ computes $\BN$. The only way to violate the inequality is if $\Theta \notin \bm{\mathcal{C}}_{S}$. But $\bm{\mathcal{C}}_{S}$ only demands that at each sum node $t$ splitting on $V_i$ the parameters are in $\mathcal{C}_{V_i \vert u_i}$, which will certainly be true if $\Theta\in\bm{\mathcal{C}}_{\mathcal{N}}$. 
\end{proof}

If we apply this to construct a CSPN for our treatment example, it will be a CSPN over the SPN in Figure \ref{fig:spnexample}, where the weights of the sum nodes are constrained by the credal sets of the CrBN defined in Section \ref{sec:treatment}.

\subsection{Solving \texorpdfstring{$\texttt{MAR}_{\max}$}{MAR\_max}}
Analogously to CrBNs, we can define the maximum marginal probability problem for CSPNs as $\texttt{MAR}_{\max}(\mathcal{S},\bm{\mathcal{C}}, e) = \max_{\Theta \in \bm{\mathcal{C}}} l_{\mathcal{S}}[\bm{\lambda}_e, \Theta]$,
where $\bm{\lambda}_e$ refers to the appropriate instantiation of the indicators for the event $e$.

While $\texttt{MAR}_{\max}$ for the AC (and BN) is intractable, we can compute it efficiently for a CSPN \cite{maua2017credal}.

\begin{proposition}
\label{prop:maxspn}
Given a Credal SPN $\mathbb{C}_{\mathcal{S}}[\prod_{i=1}^n \mathcal{C}_{i}]$, we can solve $\texttt{MAR}_{\max}$ for this family of distributions in $\mathcal{O}(\vert \mathcal{S}\vert L)$, where $L$ is an upper bound on solving $\max_{w_i}\sum_j w_{ij}c_{j}$ subject to $w_i\in C_i$. 
\end{proposition}
\begin{proof}
If we assume the maximum possible value of the children $c_1,...,c_j$ of a sum node $t_i$ are known, finding the maximum possible value of $t_i$ can be done by solving $\max_{w_i}\sum_j w_{ij}c_{j}$ subject to $w_i\in C_i$. The same is true for product nodes, with the maximum value being the product of the maximum values of its children. By induction we can find the maximum possible value of the root node through bottom-up evaluation. For details see \cite{maua2017credal}.
\end{proof}
Figure \ref{fig:spnmaximization} illustrates the computation on the CSPN compiled from the treatment example. The algorithm evaluates nodes bottom up in the graph, with the indicators set to their appropriate value ($\lambda_T=\lambda_{s_1}=\lambda_{s_2}=0$, the rest 1). The $s_1$, $s_2$ branches always lead to an indicator with the value 0. When a sum node is reached, the maximizing weights allowed by the credal set at that sum node are picked. For the left $+_V$ node this means assigning $\theta_{V\vert s_3}$ the lowest weight allowed ($0.4$), while the right $+_V$ is instead maximized with the highest weight allowed ($0.8$). No choice is made at $+_R$ since it is a singleton, and at $+_S$ the maximum weight for $s_3$ ($0.1$) is chosen. This demonstrates that $\texttt{MAR}_{\max}(S, \bm{\mathcal{C}},\bar{T}\land s_3)=0.07$.

Our overall method \textbf{CUB} is summarized in Algorithm \ref{alg:upper}. The following theorem, which follows directly from Lemma \ref{lem:bound}, shows that we do indeed return an upper bound:

\begin{algorithm}
\footnotesize

\SetAlgoLined
\KwInput{Credal Bayesian Network $\mathbb{C}_{\mathcal{N}}[\bm{\mathcal{C}}_{\mathcal{N}}]$, event $e$, order $\sigma$
}
\KwResult{Upper bound on $\marmax(\mathcal{N}, \bm{\mathcal{C}}_{\mathcal{N}}, e)$}
\Begin{
Compile $\mathcal{N}$ to AC obeying topological order $\sigma$;\\
Construct a credal SPN $\mathbb{C}_{\mathcal{S}}[\bm{\mathcal{C}}_{\mathcal{S}}]$ from the AC;\\
Compute $\texttt{MAR}_{\max}(S, \bm{\mathcal{C}}_\mathcal{S}, e)$ for this credal SPN;\\
}
\textbf{Return} $\texttt{MAR}_{\max}(S, \bm{\mathcal{C}}_\mathcal{S}, e)$
\caption{Upper Bounding for $\marmax$}
\label{alg:upper}
\end{algorithm}

\begin{theorem}

\label{thm:alg}
The
output $\texttt{MAR}_{\max}(\mathcal{S}, \bm{\mathcal{C}}_{\mathcal{S}}, e)$ returned by Algorithm \ref{alg:upper} satisfies $$\texttt{MAR}_{\max}(S, \bm{\mathcal{C}}_{\mathcal{S}}, e)\geq \texttt{MAR}_{\max}(\mathcal{N}, \bm{\mathcal{C}}_{\mathcal{N}},e)$$

\end{theorem}
Thus, we find that the probability of not assigning treatment to a patient with the severe strain in the treatment example can be no greater than $\texttt{MAR}_{\max}(\mathcal{S}, \bm{\mathcal{C}},\bar{T}\land s_3)=0.07$.

\subsection{Tightness of Upper Bound} \label{sec:tight}

Though our algorithm provides an upper bound on $\marmax$ for the Bayesian network, it will not typically be tight. This is illustrated in Figure \ref{fig:spnmaximization}, where the two different sum nodes representing $\bm{\theta}_{V|s_3}$ are assigned different weights by the maximization in the CSPN, while this is not possible for the original CrBN. We will now provide a precise characterization of the looseness of the upper bound, using the concept of structural enrichment to find an enriched CrBN which can be put in a 1-1 correspondence with the CSPN.

\begin{restatable}{definition}{defStruc}
A structural enrichment of a CrBN $\mathbb{C}_{\mathcal{N}}[\bm{\mathcal{C}}]$ is a new CrBN $\mathbb{C}_{\mathcal{N'}}[\bm{\mathcal{C}}']$ with a new underlying graph $(\bm{V},\bm{E}')$ such that $\bm{E}\subseteq \bm{E}'$, and a new credal set given by
$$(\forall \bm{w}_i \in \bm{W}_i) \mathcal{C}_{V_i \vert \bm{u}_i,\bm{w}_i}'=\mathcal{C}_{V_i \vert \bm{u}_i},$$
where $\bm{U}_i$ are the parents of $V_i$ in $\mathcal{N}$, while $\bm{W}_i$ are the newly added parents in $\mathcal{N}'$ which were not parents in $\mathcal{N}$.  
\end{restatable}

To illustrate this, suppose that we had a BN with 3 variables $A, B, C$, where $A$ is the only parent of $C$ and we have the credal set $\theta_{C=0|A=0} \in [0.3, 0.8]$. If we now consider a structurally enriched BN where $A, B$ are both parents of $C$, then we have the same interval $\theta_{C=0|A=0, B=b} \in [0.3, 0.8]$ for $b \in \{0, 1\}$, but, crucially, the parameters for $b = 0$ and $b = 1$ can take different values in this interval.
\begin{definition}
Given a CrBN $\mathbb{C}_{\mathcal{N}}[\bm{\mathcal{C}}]$ and ordering $\ordering$, the \emph{maximal structural enrichment} $\mathbb{C}_{\mathcal{N}^+_{\ordering}}[\bm{\mathcal{C}}_\ordering^+]$ is the (unique) structurally enriched CrBN which has an edge $(V_i,V_j)$ for all $i<_{\ordering}j$.
\end{definition}
The maximal structural enrichment of a CrBN with some ordering simply allows for the choice of parameters (within the credal set) at some variable to depend on all variables earlier in the order. In the case of the treatment example, the ordering $S,R,V,T$ (used for compilation in Figure \ref{fig:acexample}) would give a structurally enriched CrBN where the parameter $\bm{\theta}_{V|s_3}$ is allowed to depend on $R$, as it does in the CSPN (Figure  \ref{fig:spnmaximization}).  
\begin{restatable}{theorem}{thmExact}
\label{thm:strongbound}
The output $\texttt{MAR}_{\max}(\mathcal{S},\bm{\mathcal{C}}, e)$ returned by Algorithm \ref{alg:upper} using ordering $\ordering$ satisfies
$$\texttt{MAR}_{\max}(\mathcal{S}, \bm{\mathcal{C}}, e)=\texttt{MAR}_{\max}(\mathcal{N}^+_{\ordering}, \bm{\mathcal{C}}^+_{\ordering},e).$$

\end{restatable}
An implication of this result (see Appendix for the proof) is that the ordering $\sigma$ used when compiling the SPN can affect the tightness of the bound. Consequently, it is possible to search over topological orderings to obtain a better bound, at the cost of additional computation; we exploit this in our experiments as the method \textbf{$\text{CUB}_{max}$}. It also demonstrates that if we do, in fact, want to bound the probability of an event in a maximally ordered enrichment, then Algorithm \ref{alg:upper} will give an exact result.

We can also make use of this result to lower bound $\marmax$. We can \textit{project} the optimal parameters $\bm{\theta}^+_{V_i\vert \bm{u}_i, \bm{w}_i}$ found for $\mathbb{C}_{\mathcal{N}^+_{\ordering}}[\bm{\mathcal{C}}_\ordering^+]$ to obtain parameters $\bm{\theta}_{V_i\vert u_i}=\bm{\theta}^+_{V_i\vert \bm{u}_i, \bm{w}_i^*}$ valid for $\mathbb{C}_{\mathcal{N}}[\bm{\mathcal{C}}]$, by fixing some $w_i^*$ for each credal set. It is not guaranteed that the exact solution for $\mathbb{C}_{\mathcal{N}}[\bm{\mathcal{C}}]$ will be such a projection, but it is much easier to search over projections than parameter values and this can provide a strong lower bound in many cases, or serve as a way to initialize a more thorough search algorithm. In our experiments we will evaluate a local greedy search algorithm \textbf{CLB}, which is initialized to an arbitrary projection given by some $\bm{w}_i$ for each credal set $\bm{\mathcal{C}}_i$, and tries a series of local changes $\bm{w}_i \rightarrow \bm{w}_i'$, keeping any that increase the probability. It terminates if it reaches parity with the upper bound or no local improvement can be found. Note that there is no guarantee of convergence to the upper bound -- by Theorem \ref{thm:strongbound} it is only possible when $\texttt{MAR}_{\max}(\mathcal{N}^+_{\ordering}, \bm{\mathcal{C}}^+_{\ordering},e)=\texttt{MAR}_{\max}(\mathcal{N}, \bm{\mathcal{C}},e)$, and even when this holds \textbf{CLB} can get stuck in local optima.

\section{Experimental Results}

We evaluate our method on the CREPO \cite{cabanas2021crepo} credal inference benchmark, which consists of 960 queries over 377 small-to-moderately sized networks, and, to evaluate scalability, hepar2, a 70-node Bayesian network. We include three of our methods\footnote{Code for algorithms and experiments available at https://github.com/HjalmarWijk/credal-bound}: (i) \textbf{CUB}, which computes an upper bound; (ii) \textbf{$\text{CUB}_{max}$}, which searches over ($n= 30$) orderings to obtain a better bound; and (iii) \textbf{CLB}, which computes a lower bound as described in Section \ref{sec:tight} (capped to $n=100$ steps). 

We compare the performance of our methods to exact credal variable elimination \cite{COZMAN2000199} (where feasible) and ApproxLP \cite{antonucci2015approxlp}, an approximate method returning a lower bound which has been shown to be state of the art both in terms of scalability and accuracy of inferences. We do not consider comparison to the IntRob algorithm presented in \cite{wang2021provable} as it cannot address arbitrary credal sets, and Algorithm \ref{alg:upper} is equivalent to theirs in the limited cases where both can be applied (when all credal sets are either singletons or maximal).

We split CREPO into two subsets, CREPO-exact (768 queries), where an exact solution could be computed, and CREPO-hard (192 queries), where it ran out of memory (16GB). Since other methods do not support inferences involving decision functions, we use an augmented BN only for hepar2, where both the exact and ApproxLP methods run out of memory even without a decision function.

In Table \ref{tbl:comparison}, for all benchmarks we report the time taken by each method. For CREPO-exact, we compute the average difference in computed probability to the exact result (positive/negative for upper/lower bounds respectively), while for the other sets we report the average difference to the best upper bound. Remarkably, we see that our upper-bounding and lower-bounding algorithms dominate ApproxLP on CREPO-exact, with better lower bounds being produced in an order of magnitude less time. Given the simplicity of the greedy iteration in \textbf{CLB}, this is primarily explained by the effectiveness of projection from the upper bound as a starting heuristic. On CREPO-hard, our upper bounding is the only method capable of providing guarantees. Meanwhile, our lower bound performs worse on average than ApproxLP, but only by a small amount, while using significantly less time. Finally, we see that our method is the only one to scale to the challenging hepar2 network, completing in a reasonable amount of time even with the significant additional computational expense of incorporating a decision function.

\begin{table}
\resizebox{\columnwidth}{!}{
\begin{tabular}{lllllll}
\textbf{Network}                                                   &        & \textbf{Exact} & \textbf{ApproxLP} & \textbf{CLB} & \textbf{CUB} & \textbf{$\text{CUB}_{max}$} \\ \toprule
\multirow{2}{*}{CREPO-exact}                          & Diff   & 0     &   -0.0523    & -0.0432 & 0.0018 & 0.0015       \\
                                                          & Time(ms)   & 626        & 384     & 46(6) & 2(6) & 209(618)                        \\ \midrule
\multirow{2}{*}{CREPO-hard} & Diff & -       & -0.0529 & -0.0742 & 0.0220 & 0           \\
                                                          & Time(ms)   & -      & 1154           & 65(6) & 2(6) & 231(618)            \\ \midrule
\multirow{2}{*}{Hepar2}        & Diff & -          & -   & -0.0917  & 0 & -              \\ 
                                                          & Time (s)   & -          & - & 429(287)    & 4(287) & -    
                                                          
\end{tabular}
}
\caption{Average computation time (compilation time in parenthesis) and difference in computed probability to exact/upper bound. $-$ indicates the method ran out of memory (16GB).}
\label{tbl:comparison}
\end{table}

\section{Conclusions}
We have demonstrated how to construct
an SPN whose parameters (sum node weights) can be semantically interpreted
as representing specific conditional probability distributions
in a CrBN. The result relies on a novel SPN compilation technique, which ensures that (i) all sum nodes correspond to some variable $V$ and (ii) that the values of all parents of $V$ can be uniquely determined. This is significant as (after applying
constraint relaxation) it enables a direct mapping of the credal
sets of the CrBN to a CSPN, which, unlike the CrBN, can be
tractably maximized. This gives an efficient method to analyse robustness of decision functions learnt from data in the presence of imprecise knowledge, distributional shifts and exogenous variables.  
Our method provides formally guaranteed upper and lower bounds on the probability of an event of interest, and the experimental evaluation has additionally demonstrated that it compares favourably in accuracy to state-of-the-art approximate methods while being orders of magnitude faster.

In future work the upper bound could 
be improved through reintroducing some of the dropped equality constraints between weights of sum nodes in the CSPN, though this will involve trade-offs between computational challenge and accuracy. 
The methodology could 
also be extended to handle more challenging queries such as maximum expectations, by imposing additional structure on the compiled circuit.

\section*{Acknowledgements}
This project was funded by the ERC
under the European Union’s Horizon 2020 research and innovation programme (FUN2MODEL, grant agreement No.
834115), and by the Future of Humanity Institute, Oxford University.

\bibliographystyle{named}
\bibliography{references.bib}

\newpage
\appendix

\section{Proofs}

\subsection{Proof of Lemma \ref{lem:parents}}

\lemParents*
\begin{proof}
First, we show that the part of the subcircuit "external to" the sub-AC $\AC_{\acnode}$ is sufficient to determine the value of the parents of $V$. 

Suppose that $\subcirc, \subcirc'$ are two complete subcircuits which both include $\acnode$ and differ only in the chosen $+$-node children in the sub-AC from $t$. Since $\AC$ respects a topological ordering, by Defn. \ref{def:topord} no nodes in $\AC_{\acnode}$ split on any variable in $\text{pa}_{\mathcal{N}}(V)$. Then, these subcircuits must assign the same values to $\text{pa}_{\mathcal{N}}(V)$.

Now for the main result, suppose for contradiction there exist two complete subcircuits $\subcirc_1, \subcirc_2$ which both include $\acnode$ and specify different values $\parvals_1, \parvals_2$ for $\text{pa}_{\mathcal{N}}(V)$ through indicators. We will write $\subcirc_1 = (\subcirc_{1_P}, \subcirc_{1_S})$, where $\subcirc_{1_P}$ (prefix) is the part of the subcircuit external to $\AC_{\acnode}$, and $\subcirc_{1_S}$ (suffix) is the part internal to $\AC_{\acnode}$ (similar for $\subcirc_2$). Now we consider constructing a subcircuit $\subcirc_2' = (\subcirc_{2_P}, \subcirc_{1_S})$ which has the prefix of $\subcirc_2$, but suffix of $\subcirc_1$. 

Comparing $\subcirc_1, \subcirc_2'$, we see that they are identical in $\AC_{\acnode}$, meaning that they specify the same value of $V$ (say, $\val$), but they specify different values $\parvals_1, \parvals_2$ of $\text{pa}_{\mathcal{N}}(V)$. Since each subcircuit corresponds to a term of the network/AC polynomial, the subcircuits must contain parameters $\param_{\val|\parvals_1}, \param_{\val|\parvals_2}$ respectively. Since these parameters depend on the value of $V$, they must appear in $\AC_{\acnode}$. But this is a contradiction as both subcircuits are identical in that sub-AC.
\end{proof}

\subsection{Proof of Theorem \ref{thm:strongbound}}
\thmExact*
To prove Theorem \ref{thm:strongbound} we will introduce an expansion operation $\texttt{EX}$ on SPNs, which intuitively `distributes' all products over the sum nodes, so that product nodes only have leaf node children. As we perform the expansion we will use labels on the sum nodes to track their origin in the original SPN.
\begin{definition}
\label{def:ex}
Let $\SPN$ be a SPN respecting the ordering $\sigma$, where each sum node is given a unique label. We define the \emph{expanded SPN} $\EX(\SPN)$ to be an SPN constructed as follows:

\begin{itemize}
    \item If the root $t$ is a single leaf node, then $\EX(\SPN)=\SPN$;
    \item If the root $t$ is a sum node labelled $l$ with children $t_1,...,t_n$, then $\texttt{EX}(\mathcal{S})$ is a sum node also labelled $l$ with children $\texttt{EX}(\mathcal{S}_{t_1}),...,\texttt{EX}(\mathcal{S}_{t_n})$ and the same weights;
    \item If the root is a product node which has only leaf nodes as children, then $\texttt{EX}(\mathcal{S})=\mathcal{S}$;
    \item If the root is a product node $t$ with at least one product node child $p$, then $\texttt{EX}(\mathcal{S})$ is the result of applying $\texttt{EX}$ to a single product node with all the children of $t$ and $p$ together;
    \item If the root is a product node $t$ with at least one sum node child and no product node children, then let $s$ (labelled $l$) be the first sum node child in $\sigma$, let $s_1,...,s_n$ be the children of $s$ and let $t_1,...,t_m$ be the other children of $t$. Then $\texttt{EX}(\mathcal{S})$ is a sum node (labelled $l$) with $n$ children $\texttt{EX}(p_1),...,\texttt{EX}(p_n)$, and weights identical to $s$. Each SPN $p_i$ has a product node at the root and children $t_1,...,t_m,s_i$, labelled with their original labels.
\end{itemize}
\end{definition} 

It should be noted that each of these operations performs recursive calls on SPNs with fewer nodes than the original SPN. Thus the recursion will always finish.
\begin{lemma}
For any SPN $\mathcal{S}$ respecting ordering $\sigma$, the expanded SPN $\EX(\SPN)$ has the same SPN polynomial $l_{\mathcal{S}}[\bm{\lambda},\bm{\Theta}]=l_{\texttt{EX}(\mathcal{S})}[\bm{\lambda},\bm{\Theta}]$. Further, $\EX(\SPN)$ also respects $\sigma$, and all its product nodes have only leaf nodes as children.

\end{lemma}
\begin{proof}
We will prove this by induction on the SPNs $\EX(\SPN_t)$, where we proceed in reverse topological order of the nodes $t$ in the SPN (i.e. children before parents).

First, if the root is a leaf node or a product node with only leaf node children then all the statements are trivially true.

If the root is a sum node $t$ labelled $l$ with children $t_1,...,t_n$, then the SPN polynomial is the weighted sum of the polynomials of the expanded children which are unchanged (by the induction hypothesis), where the weights (parameters) are the same as in the original SPN. Further, the expanded children split on the same variables as before, respecting ordering, so the sum node as a whole does so as well.

If the root is a product node with another product node as a child, then it is enough to observe that (by the associativity of products) moving all children to one product node has no effect on the polynomial or ordering - the remaining result follows from the inductive hypothesis on the combined product node.

It remains to show the result for product nodes with at least one sum node child (and no product node children). Adopting the same notation used in the definition, we will show that the sum node $s$ must split on a variable $V_i$ such that $V_i < V_j$ for all other variables $V_j$ split on somewhere in $\mathcal{S}$. For $V_k$ split on at one of the other sum node children $t_1,...,t_m$ we have $V_i<V_k$ by construction. For some $V_j$ split on elsewhere any node splitting on it must be a descendant of either $s$ or some $t_k$, and since $\mathcal{S}$ obeys the ordering condition $V_k < V_j$ which shows the claim by transitivity. This (along with the inductive hypothesis) shows that ordering is respected. The polynomial being the same follows immediately from the distributive law.

The overall procedure only produces sum nodes, leaf nodes, and product nodes with only leaf nodes as children.
\end{proof}
Note that requiring all product nodes to have only leaf node children is a very restrictive condition. Any complete subcircuit of such an SPN must be a single path of sum nodes $t_1,...,t_n$ followed by a single product node containing all its leaf nodes at the end. We can now prove Theorem \ref{thm:strongbound}.
\begin{proof}
Given the CrBN $\mathbb{C}_{\mathcal{N}}[\bm{\mathcal{C}}_{\mathcal{N}}]$ (with maximal ordered enrichment $\mathbb{C}_{\mathcal{N}^+_{\ordering}}[\bm{\mathcal{C}}^+_{\mathcal{N}}]$), let $\mathbb{C}_{\mathcal{S}}[\bm{\mathcal{C}}_{\mathcal{S}}]$ be the CSPN constructed by Algorithm \ref{alg:upper} respecting ordering $\ordering$.

We now define credal sets $\CREDAL^-_{\mathcal{S}}$ and $\CREDAL^+_{\mathcal{S}}$ over the weights/parameters of the \textit{expanded} SPN $\EX (\SPN)$. In both, for a sum node $s$ in $\EX(\SPN)$, we assign individual credal sets $\mathcal{C}_i$ to the weights of $s$ from the node with the same label in the original SPN $\SPN$. However, for $\CREDAL^-_{\mathcal{S}}$, we additionally impose the constraint that sum nodes with the same label in $\EX(\SPN)$ must also have the same weights; $\mathbb{C}_{\EX (\mathcal{S})}[\bm{\mathcal{C}}^-_{\mathcal{S}}]$ is thus not a CSPN (while $\mathbb{C}_{\EX (\mathcal{S})}[\bm{\mathcal{C}}^+_{\mathcal{S}}]$ is).

Firstly, we show that $\marmax(\SPN, \CREDAL_{\SPN} \evidence) = \marmax(\EX (\SPN), \CREDAL^-_{\SPN}, \evidence)$. Since $\SPN$ and $\EX(\SPN)$ have the same polynomial, the distributions are equivalent (for the same parameters). Further, since the credal sets $\CREDAL_{\SPN}, \CREDAL_{\SPN}^-$ are the same by construction, the maximal marginal probability (and parameter assignment) is the same for both.

Second, we show that $\marmax(\EX (\SPN), \CREDAL^+_{\SPN}, \evidence) = \marmax(\EX (\SPN), \CREDAL^-_{\SPN}, \evidence)$, that is, that the same value can be obtained despite the additional constraints on $\CREDAL^-_{\SPN}$. Recall that, since $\mathbb{C}_{\EX (\mathcal{S})}[\bm{\mathcal{C}}^+_{\mathcal{S}}]$ is a CSPN, we can find $\marmax(\EX (\SPN), \CREDAL^+_{\SPN}, \evidence)$ by maximizing the weights at each sum node locally based on the maximized value of its children (\ref{prop:maxspn}).  Looking at the definition of the expansion process (Defn. \ref{def:ex}), we see that the only way in which two sum nodes $r, r'$ in $\EX(\SPN)$ can have the same label is if two product nodes in the original SPN have a common sum node child (with that label). Then, using similar notation as in the Definition, $r$ will have children $\EX(p_1), ..., \EX(p_n)$, where $p_i$ is an SPN with a product node root and children $t_1, ..., t_m, s_i$, while $r'$ will have children $\EX(p_1'), ..., \EX(p_n')$ where $p_i'$ is an SPN with product node root and children $t_1', ..., t_{m'}', s_i$. Since applying $\EX$ does not change the SPN polynomial, the $i^{\text{th}}$ child of $r$ has polynomial $l_{\SPN_{p_i}}[\indicators, \parameters] = l_{\SPN_{s_i}}[\indicators, \parameters] \prod_{j=1}^{m} l_{\SPN_{t_i}}[\indicators, \parameters]$. Since SPN polynomials are multilinear functions of their parameters, each parameter can appear in only one of the RHS polynomials. Thus we can maximize each separately, i.e. $\max_{\parameters \in \CREDAL^+_{\SPN}} l_{\SPN_{p_i}}[\indicators, \parameters] = \max_{\parameters \in \CREDAL^+_{\SPN}} l_{\SPN_{s_i}}[\indicators, \parameters] \prod_{j=1}^{m} \max_{\parameters \in \CREDAL^+_{\SPN}} l_{\SPN_{t_j}}[\indicators, \parameters] = \max_{\parameters \in \CREDAL^+_{\SPN}} l_{\SPN_{s_i}}[\indicators, \parameters] \prod_{j=1}^{m} c$ where $c := \prod_{j=1}^{m} \max_{\parameters \in \CREDAL^+_{\SPN}} l_{\SPN_{t_j}}[\indicators, \parameters]$ is independent of $i$. Applying a similar argument to the $i^{\text{th}}$ child of $r'$, we get $\max_{\parameters \in \CREDAL^+_{\SPN}} l_{\SPN_{p_i'}}[\indicators, \parameters] = \max_{\parameters \in \CREDAL^+_{\SPN}} l_{\SPN_{s_i}}[\indicators, \parameters] \prod_{j=1}^{m} c'$ where $c' := \prod_{j=1}^{m'} \max_{\parameters \in \CREDAL^+_{\SPN}} l_{\SPN_{t_j'}}[\indicators, \parameters]$ is again independent of $i$. Thus we see that the value of the $i^{\text{th}}$ child of $r$ and $r'$ are proportional, and hence the same choice of parameters/weights for $r, r'$ maximize both. Thus we have that the maximizing weight in $\CREDAL^+_{\SPN}$ is also in $\CREDAL^-_{\SPN}$, and $\marmax(\EX (\SPN), \CREDAL^+_{\SPN}, \evidence) = \marmax(\EX (\SPN), \CREDAL^-_{\SPN}, \evidence)$.

Finally, we want to show that $\marmax(\EX (\SPN), \CREDAL^+_{\SPN}, \evidence) = \texttt{MAR}_{\max}(\mathcal{N}_{\ordering}^+,\bm{\mathcal{C}}_{\mathcal{N}}^+,e)$. We have already established that the product nodes in $\EX(\SPN)$ have only leaf node children. Thus, each product node must correspond to a particular instantiation of the variables, and the sum nodes from the root must branch out to cover all of these instantiations. Since $\EX_{\SPN}$ respects the ordering $\ordering$, this means that $\EX(\SPN)$ takes the form of a rooted tree which splits on each variable in succession in the order $\ordering$. This also provides a very clear semantics for the weight on each edge of a sum node $t$: it is simply $\theta_{v_i|u_i}$, where $v_i$ is the value corresponding to that edge of the variable $V_i$ which $t$ is splitting on, and $u_i$ records the values of all variables before $V_i$ in the ordering $\ordering$ (which is unique as the SPN is a tree). It is then apparent that $\EX(\SPN)$ precisely describes the polynomial of $\mathbb{C}_{\mathcal{N}^+_{\ordering}}[\bm{\mathcal{C}}^+_{\mathcal{N}}]$. Further, the credal sets $\CREDAL^+_{\mathcal{N}}$ and $\CREDAL^+_{\SPN}$ are the same by construction, so we are done.

All together, we have that $\marmax(\SPN, \CREDAL_{\SPN} \evidence) = \marmax(\EX (\SPN), \CREDAL^-_{\SPN}, \evidence) = \marmax(\EX (\SPN), \CREDAL^+_{\SPN}, \evidence) = \texttt{MAR}_{\max}(\mathcal{N}_{\ordering}^+,\bm{\mathcal{C}}_{\mathcal{N}}^+,e)$, as required.
\end{proof}

\section{Comparison to Intervention Sets}

%We wish to study settings where there is some uncertainty in the data-generating process, forcing us to consider performance across a whole family of possible distributions. A salient set of distribution families \cite{wang2021provable} are those which arise from allowing certain interventions on the Bayesian network.
Bayesian networks are particularly convenient for studying data-generating processes which include some adversary or unknown process that could seemingly arbitrarily \emph{intervene} on the behaviour of certain variables in our system, thus changing the distribution. \cite{wang2021provable} formalize this by defining intervention sets where some subset of variables $\bm{W} \subseteq V$. In this section, we show how our formulation of credal sets over augmented BNs generalizes their definitions.

\cite{wang2021provable} consider \emph{parametric interventions} on variables $\bm{W}$ in a BN $\mathcal{N}$ which assign new values to all the parameters associated with variables in $\bm{W}$. Each such intervention leads to a new BN $\mathcal{N}'$ with a new joint distribution $p_{\mathcal{N}'}$. We define $\mathcal{I}_N[\bm{W}]$ to be the family of distributions arising from some parametric intervention on $\bm{W}$ in $\mathcal{N}$. A \emph{structural intervention} on $\bm{W}$ in $\mathcal{N}$ can further introduce new edges to the graph through a context function $C_{\bm{W}}:V\to\mathcal{P}(V)$, allowing the variables in $\bm{W}$ to depend on variables which previously they could not. We define $\mathcal{I}_N[\bm{W},C_{\bm{W}}]$ to be the family of distributions arising from some structural intervention with context function $C_{\bm{W}}$ on $\bm{W}$ in $\mathcal{N}$.

\subsection{Credal Sets}
The interventional families of distributions assume that for each variable we either have complete certainty about its behaviour (if it is not an intervenable variable), or its behaviour is completely unknown (if it is an intervenable variable). This does not allow us to represent other forms of uncertainty, such as that which might arise from learning a model from data. However, these can be represented by credal sets.

\begin{proposition}
\label{prop:param}
For any BN $\mathcal{N}$ and subset of the variables $\bm{W}$, there exists a CrBN $\mathbb{C}_{\BN}[\CREDAL]$ such that
$$\mathcal{I}_{\mathcal{N}}[\bm{W}]=\mathbb{C}_{\BN}[\CREDAL].$$
\end{proposition}
\begin{proof}
Given $\bm{W}$, construct a CrBN $\mathbb{C}_{\BN}[\CREDAL]$ with a credal set such that $\mathcal{C}_{V_i\vert u_i}$ is maximal (all probability distributions allowed) if $V_i\in \bm{W}$, and a singleton containing only the parameter value in $\mathcal{N}$ otherwise. For any BN $\mathcal{N}'\in \mathcal{I}_\mathcal{N}[\bm{W}]$ it will differ from $\mathcal{N}$ only for parameters related to variables in $\bm{W}$. Since these are allowed to take any value by the credal set  $\CREDAL$, we must have $\mathcal{N}'\in \mathbb{C}_{\BN}[\CREDAL]$. Likewise, for any $\mathcal{N}'\in \mathbb{C}_{\BN}[\CREDAL]$ it can differ from $\mathcal{N}$ only for parameters of variables in $\bm{W}$, so there will be some intervention generating $\mathcal{N}'$. 
\end{proof}
This demonstrates that the notion of CrBNs generalizes the notion of parametric interventions. It can also generalize structural interventions, though we need a CrBN on a structurally enriched graph. Here we reproduce our definition of structural enrichments for CrBNs from the main paper:
\defStruc*

\begin{proposition}
For any BN $\mathcal{N}$, subset of the variables $\bm{W}$, ordering $\sigma$, and context function \cite{wang2021provable} $C_{\bm{W}}$ whose edges respect $\sigma$, there exists a structurally enriched CrBN $\mathbb{C}_{\BN}[\CREDAL]$ such that
$$\mathcal{I}_{\mathcal{N}}'[\bm{W},C_{\bm{W}}]=\mathbb{C}_{\BN}[\CREDAL].$$
\end{proposition}
\begin{proof}
We construct $\mathbb{C}_{\BN'}[\CREDAL']$ from $\mathcal{N}$ and $\bm{W}$ as in the proof of Proposition \ref{prop:param}. We then construct the structural enrichment $\mathbb{C}_{\BN}[\CREDAL]$ of $\mathbb{C}_{\BN'}[\CREDAL']$, where we add all edges given by the context function $C_{\bm{W}}$ (guaranteed to be possible since the context function is compatible with $\sigma$). For $V_i\in \bm{W}$, after observing any value for the newly added parents, the credal sets for any value of the remaining parents must be maximal (allow any probability distribution) by construction, and so the credal sets must be maximal for all values of all parents. For $V_i \notin \bm{W}$, after observing any value for the newly added parents the value must be a singleton for all values of the old parents by construction, so it must be a singleton for all values of all parents. As in Proposition \ref{prop:param}, we thus have $\mathbb{C}_{\BN}[\CREDAL]=\mathcal{I}_{\mathcal{N}}'[\bm{W},C_{\bm{W}}]$. 
\end{proof}

This demonstrates that CrBNs fully generalize families found by allowing causal interventions of the type in \cite{wang2021provable}. 

\section{Experimental Details}
\subsection{CREPO}
The full list of queries and models included in the benchmark is given at \hyperlink{https://github.com/IDSIA/crepo}{https://github.com/IDSIA/crepo}. The split into exact and hard is post-hoc and per-query, based on which queries ran out of memory (16GB) when running credal variable elimination. 
\subsection{hepar2}
Since there is to our knowledge no existing credal sets over the hepar2 parameters (other than the ones in \cite{wang2021provable}, where our generalized algorithm exactly matches their results, as expected), we create credal sets %ad hoc 
by allowing perturbations of $\epsilon=0.1$ to all parameters. The results are averaged over 100 randomized queries; the full list can be found at \hyperlink{https://github.com/HjalmarWijk/credal-bound}{https://github.com/HjalmarWijk/credal-bound} 
\subsection{Credal Set Representation}
As ApproxLP is based on solving linear programming instances, it requires credal sets to be represented as a set of linear constraints. Meanwhile, credal variable elimination traditionally assumes the credal set is specified by a finite set of vertices. Our algorithm is entirely agnostic regarding the credal set representation, but in order to allow comparisons with both of these approaches we use Polco (through CREMA \cite{huber2020a}) to convert between them.
\
\end{document}